\documentclass[english]{article}
\usepackage[T1]{fontenc}
\usepackage[utf8]{inputenc}
\usepackage{array}
\usepackage{wrapfig}
\usepackage{calc}
\usepackage{tablefootnote}
\usepackage{amsmath}
\usepackage{amsthm}
\usepackage{amssymb}
\usepackage{graphicx}
\usepackage[authoryear]{natbib}
\usepackage{microtype}

\makeatletter

\providecommand{\tabularnewline}{\\}

\theoremstyle{plain}
\newtheorem{thm}{\protect\theoremname}
\theoremstyle{plain}
\newtheorem{lem}[thm]{\protect\lemmaname}
\ifx\proof\undefined
\newenvironment{proof}[1][\protect\proofname]{\par
	\normalfont\topsep6\p@\@plus6\p@\relax
	\trivlist
	\itemindent\parindent
	\item[\hskip\labelsep\scshape #1]\ignorespaces
}{%
	\endtrivlist\@endpefalse
}
\providecommand{\proofname}{Proof}
\fi

\usepackage[final]{neurips_2019}
\usepackage{hyperref}       % hyperlinks
\usepackage{url}            % simple URL typesetting
\usepackage{booktabs}       % professional-quality tables
\usepackage{amsfonts}       % blackboard math symbols
\usepackage{nicefrac}       % compact symbols for 1/2, etc.
\usepackage{printlen}

\makeatother

\usepackage{babel}
\providecommand{\lemmaname}{Lemma}
\providecommand{\theoremname}{Theorem}

\begin{document}

\title{On the Curved Geometry of Accelerated Optimization}

\author{Aaron Defazio \\
 Facebook AI Research\\
New York}
\maketitle
\begin{abstract}
In this work we propose a differential geometric motivation for Nesterov's
accelerated gradient method (AGM) for strongly-convex problems. By
considering the optimization procedure as occurring on a Riemannian
manifold with a natural structure, The AGM method can be seen as the
proximal point method applied in this curved space. This viewpoint
can also be extended to the continuous time case, where the accelerated
gradient method arises from the natural block-implicit Euler discretization
of an ODE on the manifold. We provide an analysis of the convergence
rate of this ODE for quadratic objectives.
\end{abstract}

\section{Introduction\vspace{-0.5em}
}

The core algorithms of convex optimization are gradient descent (GD)
and the accelerated gradient method (AGM). These methods are rarely
used directly, more often they occur as the building blocks for distributed,
composite, or non-convex optimization. In order to build upon these
components, it is helpful to understand not just \emph{how} they work,
but \emph{why.} The gradient method is well understood in this sense.
It is commonly viewed as following a direction of steepest descent
or as minimizing a quadratic upper bound. These interpretations provide
a motivation for the method as well as suggesting a potential convergence
proof strategy. 

The accelerated gradient method in contrast has an identity crisis.
Its equational form is remarkably malleable, allowing for many different
ways of writing the same updates. We list a number of these forms
in Table \ref{tab:nes}. Nesterov's original motivation for the AGM
method used the concept of estimate sequences. Unfortunately, estimate
sequences do not necessarily yield the simplest accelerated methods
when generalized, such as for the composite case (\citealt{fista,nes-composite}),
and they have not been successfully applied in the important finite-sum
(variance reduced) optimization setting.

Because of the complexity of estimate sequences, the AGM method is
commonly motivated as a form of momentum. This view is flawed as a
way of introducing the AGM method from first principles, as the most
natural way of using momentum yields the heavy ball method instead:
\[
x^{k+1}=x^{k}-\gamma\nabla f\left(x^{k}\right)+\beta\left(x^{k}-x^{k-1}\right),
\]
which arises from discretizing the physics of a particle in a potential
well with additional friction. The heavy-ball method does not achieve
an accelerated convergence rate on general convex problems, suggesting
that momentum, \emph{per se}, is not the reason for acceleration.
Another contemporary view is the linear-coupling interpretation of
\citet{AllenOrecchia2017}, which views the AGM method as an interpolation
between gradient descent and mirror descent. We take a more geometric
view in our interpretation.

In this work we motivate the AGM by introducing it as an application
of the proximal-point method:
\[
x^{k}=\arg\min_{x}\left\{ f(x)+\frac{\eta}{2}\left\Vert x-x^{k-1}\right\Vert ^{2}\right\} .
\]

The proximal point (PP) method is perhaps as foundational as the gradient
descent method, although it sees even less direct use as each step
requires solving a regularized subproblem, in contrast to the closed
form steps for GD and AGM. The PP method gains remarkable convergence
rate properties in exchange for the computational difficulty, including
convergence for any positive step-size. 

We construct the AGM by applying a dual form of the proximal point
method in a curved space. Each step follows a geodesic on a manifold
in a sense we make precise in Section \ref{sec:bregman-prox}. We
use the term curved with respect to a coordinate system, rather than
a coordinate free notion of curvature such as the Riemannian curvature.
We first give a brief introduction to the concepts from differential
geometry necessary to understand our motivation. The equational form
that our argument yields is much closer to those that have been successfully
applied in practice, particularly for the minimization of finite sums
\citep{lan2017,spdc}. 

\section{Connections\vspace{-0.5em}
}

A\emph{n (affine) connection} is a type of structure on a manifold
that can be used to define and compute geodesics. Geodesics in this
sense represent curves of zero acceleration. These geodesics are more
general concepts than Riemannian geodesics induced by the Riemannian
connection, not necessarily representing the shortest path in any
metric. Indeed, we will define multiple connections on the same manifold
that lead to completely different geodesics. 

Given a $n$ dimensional coordinate system, a connection is defined
by $n^{3}$ numbers at every point $x$ on the manifold, called the
connection coefficients (or Christoffel symbols) $\varGamma_{ij}^{k}(x)$.
A geodesic is a path $\gamma:[0,1]\rightarrow\mathcal{M}$ (in our
case $\mathcal{M}=\mathbb{R}^{n}$) between two points $x$ and $y$
can then be computed as the unique solution $\gamma(t)=x(t)$ to the
system of ordinary differential equations \citep[Page 58, Eq 4.11]{riemannian-book}:
\[
\frac{d^{2}\gamma^{i}}{dt^{2}}\doteq\frac{d^{2}x^{i}}{dt^{2}}+\sum_{j,k}\varGamma_{jk}^{i}(x)\frac{dx^{j}}{dt}\frac{dx^{k}}{dt}=0,
\]
with boundary conditions $x(0)=x$ and $x(1)=y.$ Here $x^{i}$ denotes
the $i\text{th}$ component of $x$ expressed in the same coordinate
system as the connection.

\section{Divergences induce Hessian manifold structure\vspace{-0.5em}
}

\label{sec:manifold-structure}Let $\phi$ be a smooth strongly convex
function defined on $\mathbb{R}^{n}$. The Bregman divergence generated
by $\phi$: 
\[
B_{\phi}(x,y)=\phi(x)-\phi(y)-\bigl\langle\nabla\phi(y),x-y\bigr\rangle,
\]
 and its derivatives can be used to define a remarkable amount of
structure on the domain of $\phi$. In particular, we can define a
\emph{Riemannian manifold, together with two dually flat connections
with biorthogonal coordinate} systems \citep{infogeombook,hessianstructuresbook}.
This structure is also known as a Hessian manifold. Topologically
it is $\mathcal{M}=\mathbb{R}^{n}$ with the following additional
geometric structures.

\subsection*{Riemannian structure\vspace{-0.5em}
}

Riemannian manifolds have the additional structure of a \emph{metric}
\emph{tensor} (a generalized dot-product), defined on their tangent
spaces. We denote the vector space of tangent vectors at a point $x$
as $T_{x}\mathcal{M}$. If we express the tangent vectors with respect
to the Euclidean basis, the metric at a point $x$ is a quadratic
form with the Hessian matrix $H(x)=\nabla_{x}^{2}B(x,y)=\nabla^{2}\phi(x)$
of $\phi$ at $x$:
\[
g_{x}(u,v)=u^{T}H(x)v.
\]

\subsection*{Biorthogonal coordinate systems\vspace{-0.5em}
}

Central to the notion of a manifold is the invariance to the choice
of coordinate system. We can express a point on the manifold as well
as a point in the tangent space using any coordinate system that is
most convenient. Of course, when we wish to perform calculations on
the manifold we must be careful to express all quantities in that
coordinate system. Euclidean coordinates $e_{i}$ are the most natural
on our Hessian manifold, however there is another coordinate system
which is naturally dual to $e_{i},$ and ties the manifold structure
directly to duality theory in optimization.

Recall that for a convex function $\phi$ we may define the convex
conjugate $\phi^{*}(y)=\max_{x}\left\{ \left\langle x,y\right\rangle -\phi(x)\right\} .$
The dual coordinate system we define simply identifies each point
$x$, when expressed in Euclidean (``primal'') coordinates, with
the vector of ``dual'' coordinates:
\[
y=\nabla\phi(x).
\]
Our assumptions of smoothness and strong convexity imply this is a
one-to-one mapping, with inverse given by $x=\nabla\phi^{*}(y).$
The remarkable fact that the gradient of the conjugate is the inverse
of the gradient is a key building block of the theory in this paper. 

The notion of \emph{biorthogonality} refers to natural tangent space
coordinates of these two systems. A tangent vector $v$ at a point
$x$ can be converted to a vector $u$ of dual (tangent space) coordinates
using matrix multiplication with the Hessian \citep{hessianstructuresbook}:
\begin{equation}
u=H(x)v,\label{eq:coc-tangent}
\end{equation}
Given the definition of the metric above, it is easy to see that if
we have two vectors $v_{1}$ and $v_{2}$, we may express $v_{2}$
in dual coordinates $u_{2}$ so that the metric tensor takes the simple
form:
\[
g_{x}(v_{1},v_{2})=v_{1}^{T}H(x)v_{2}=v_{1}^{T}H(x)\left(H(x)^{-1}u_{2}\right)=v_{1}^{T}u_{2},
\]
which is the \emph{biorthogonal} relation between the two tangent
space coordinate systems.

\subsection*{Dual Flat Connections\vspace{-0.5em}
}

\begin{table}
\caption{\label{tab:nes}\textsf{\small{}Equivalent forms of Nesterov's method
for $\mu$-strongly convex, $L$-smooth $f$. Proofs of the stated
relations are available in the appendix.}}

{\small{}}%
\begin{tabular*}{1\columnwidth}{@{\extracolsep{\fill}}|>{\centering}m{2.5cm}|c|c|}
\hline 
{\small{}Form Name} & {\small{}Algorithm} & {\small{}Relations}\tabularnewline
\hline 
\hline 
{\small{}\citet{nesterov2013introductory}}\\
{\small{}form I} & {\small{}$\begin{aligned}y^{k} & =\frac{\alpha\gamma v^{k}+\gamma x^{k}}{\alpha\mu+\gamma}\\
x^{k+1} & =y^{k}-\frac{1}{L}\nabla f(y^{k}),\\
v^{k+1} & =\left(1-\alpha\right)v^{k}+\frac{\alpha\mu}{\gamma}y^{k}-\frac{\alpha}{\gamma}\nabla f(y^{k})
\end{aligned}
$} & {\small{}$\begin{aligned}\alpha_{\text{Nes}} & =\sqrt{\mu/L}\\
\gamma_{\text{Nes}} & =\mu.
\end{aligned}
$}\tabularnewline
\hline 
{\small{}\citet{nesterov2013introductory}}\\
{\small{}form II} & {\small{}$\begin{aligned}x^{k+1} & =y^{k}-\frac{1}{L}\nabla f(y^{k}),\\
y^{k+1} & =x^{k+1}+\beta\left(x^{k+1}-x^{k}\right)
\end{aligned}
$} & {\small{}$\beta_{\text{Nes}}=\frac{\sqrt{L}-\sqrt{\mu}}{\sqrt{L}+\sqrt{\mu}}$}\tabularnewline
\hline 
{\small{}\citet{sutskever2013}} & {\small{}$\begin{aligned}p^{k+1} & =\beta p^{k}-\frac{1}{L}\nabla f\left(x^{k}+\beta p^{k}\right),\\
x^{k+1} & =x^{k}+p^{k+1}
\end{aligned}
$} & {\small{}$\begin{aligned}p_{\text{Sut}}^{k+1} & =x_{\text{Nes}}^{k+1}-x_{\text{Nes}}^{k},\\
y_{\text{Nes}}^{k} & =x_{\text{Sut}}^{k}+\beta p_{\text{Sut}}^{k}.
\end{aligned}
$}\tabularnewline
\hline 
{\small{}Modern Momentum}\tablefootnote{PyTorch \& Tensorflow (for instance) implement this form. Evaluating
the gradient and function at the current iterate $x^{k}$ instead
of a shifted point makes it more consistent with gradient descent
when wrapped in a generic optimization layer.} & {\small{}$\begin{aligned}p^{k+1} & =\beta p^{k}+\nabla f(x^{k}),\\
x^{k+1} & =x^{k}-\frac{1}{L}\left(\nabla f(x^{k})+\beta p^{k+1}\right).
\end{aligned}
$} & {\small{}$\begin{aligned}x_{\text{mod}}^{k} & =x_{\text{Sut}}^{k}+\beta p_{\text{Sut}}^{k}=y_{\text{Nes}}^{k},\\
p_{\text{mod}}^{k} & =-Lp_{\text{Sut}}^{k}.
\end{aligned}
$}\tabularnewline
\hline 
{\small{}\citet{auslender2006}} & {\small{}$\begin{aligned}y^{k} & =(1-\theta)\hat{x}^{k}+\theta z^{k},\\
z^{k+1} & =z^{k}-\frac{\gamma}{\theta}\nabla f(y^{k}),\\
\hat{x}^{k} & =(1-\theta)\hat{x}^{k}+\theta z^{k+1}.
\end{aligned}
$} & {\small{}$\begin{aligned}\theta_{\text{AT}} & =1-\beta_{\text{Nes}},\\
\hat{x}_{\text{AT}}^{k} & =x_{\text{Nes}}^{k},\\
y_{\text{AT}}^{k} & =y_{\text{Nes}}^{k}=x_{\text{mod}}^{k},\\
\gamma_{\text{AT}} & =1/L.
\end{aligned}
$}\tabularnewline
\hline 
{\small{}\citet{lan2017}} & {\small{}$\begin{aligned}\tilde{x}^{k} & =\alpha(x^{k-1}-x^{k-2})+x^{k-1},\\
\underline{x}^{k} & =\frac{\tilde{x}^{k}+\tau\underline{x}^{k-1}}{1+\tau},\\
g^{k} & =\nabla f(\underline{x}^{k}),\\
x^{k} & =x^{k-1}-\frac{1}{\eta}g^{k}.
\end{aligned}
$} & {\small{}$\begin{aligned}x_{\text{Lan}}^{k} & =z_{\text{AT}}^{k},\\
\underline{x}_{\text{Lan}}^{k} & =y_{\text{AT}}^{k},\\
\eta_{\text{Lan}} & =\frac{\gamma_{\text{AT}}}{\theta_{\text{AT}}},\\
\tau_{\text{Lan}} & =\frac{1-\theta_{\text{AT}}}{\theta_{\text{AT}}},\\
\alpha_{\text{Lan}} & =1-\theta_{\text{AT}}.
\end{aligned}
$}\tabularnewline
\hline 
\end{tabular*}\vspace{-0.5em}
\end{table}
\label{subsec:flat}There is an obvious connection $\varGamma^{(E)}$
we can apply to the Hessian manifold, the Euclidean connection that
trivially identifies straight lines in $\mathbb{R}^{n}$ as geodesics.
Normally when we perform gradient descent in $\mathbb{R}^{n}$ we
are implicitly following a geodesic of this connection. The connection
coefficients $\varGamma_{ij}^{(E)k}$ are all zero when this connection
is expressed in Euclidean coordinates. A connection that has $\varGamma_{ij}^{k}=0$
with respect to some coordinate system is a \emph{flat} connection. 

The Hessian manifold admits another flat connection, which we will
call the dual connection, as it corresponds to straight lines in the
dual coordinate system established above. In particular each dual
geodesic can be expressed in primal coordinates $\gamma(t)$ as a
solution to the equation:
\[
\nabla\phi\left(\gamma(t)\right)=at+b,
\]
for vectors $a$, $b$ representing the initial velocity and point
respectively (both represented in dual coordinates) that depend on
the boundary conditions. This is quite easy to solve using the relation
$\nabla\phi^{-1}=\nabla\phi^{*}$ discussed above. For instance, a
geodesic $\gamma:[0,1]\rightarrow\mathcal{M}$ between two arbitrary
points $x$ and $y$ under the dual connection could be computed explicitly
in Euclidean coordinates as:
\[
\gamma(t)=\nabla\phi^{*}\left(t\nabla\phi(y)+\left(1-t\right)\nabla\phi(x)\right).
\]

If we instead know the initial velocity we can find the endpoint with:
\begin{equation}
y=\nabla\phi^{*}\left(\nabla\phi(x)+H(x^{k})v\right).\label{eq:exp_geo}
\end{equation}
The flatness of the dual connection $\varGamma^{(D)}$ is crucial
to its computability in practice. If we instead try to compute the
geodesic in Euclidean coordinates using the geodesic ODE, we have
to work with the connection coefficients which involve third derivatives
of $\phi$ (taking the form of double those of the Riemannian connection
$\varGamma^{(R)}$):
\[
\varGamma_{ij}^{(D)k}(x)=2\varGamma_{ij}^{(R)k}=\left[H(x)^{-1}\left(\nabla H(x)\right)_{i}\right]_{kj},
\]
The Riemannian connection's geodesics are similarly difficult to compute
directly from the ODE (they also can't generally be expressed in a
simpler form).
\begin{figure}
\begin{centering}
\hfill{}\includegraphics[width=0.48\columnwidth]{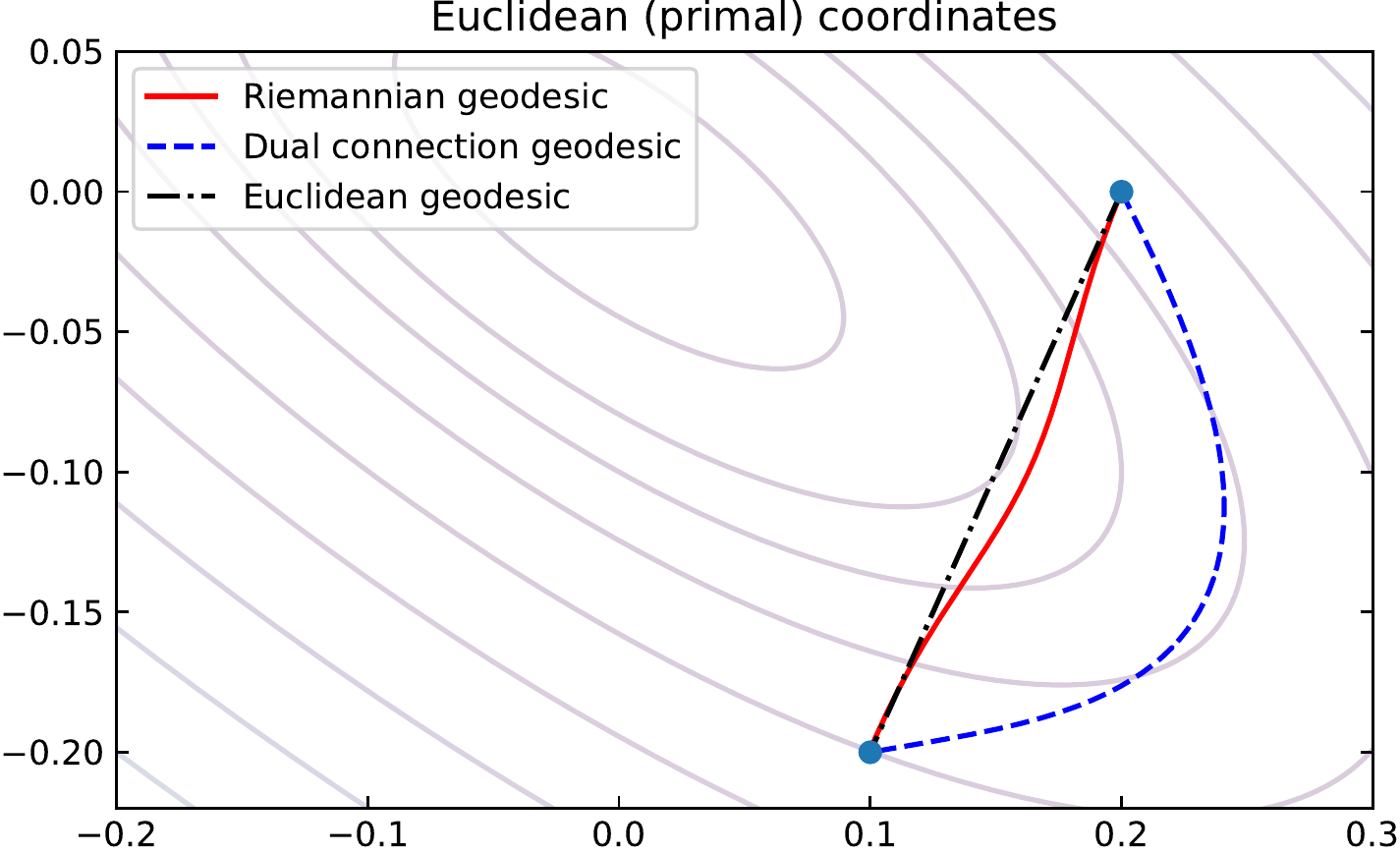}\hfill{}\includegraphics[width=0.48\columnwidth]{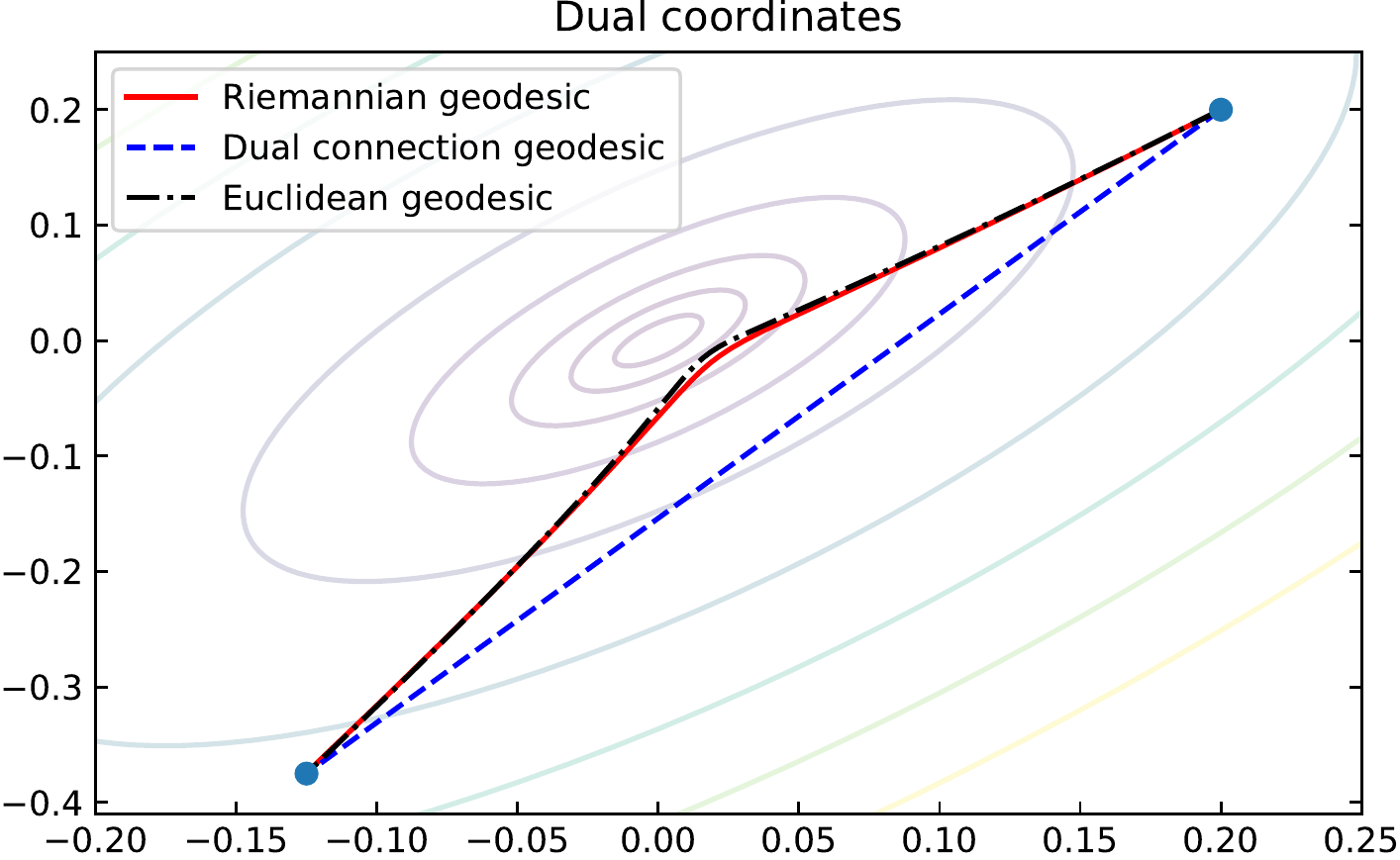}\hfill{}
\par\end{centering}
\caption{{\small{}Illustrative geodesics for $f(x)=\frac{1}{4}\left\Vert Ax\right\Vert ^{4}$
, with $A=[2,\,1;\,1,\,3].$ Viewing them from both coordinate systems
highlights the duality. Contour lines are for $f$ and $f^{*}$ respectively.}}
\end{figure}

\section{Bregman proximal operators follow geodesics\vspace{-0.5em}
}

\label{sec:bregman-prox}Bregman divergences arise in optimization
primarily through their use in proximal steps. A Bregman proximal
operation balances finding a minimizer of a given function $f$ with
maintaining proximity to a given point $y$, measured using a Bregman
divergence instead of a distance metric:
\begin{equation}
x^{k}=\arg\min_{x}\left\{ f(x)+\rho B_{\phi}(x,x^{k-1})\right\} .\label{eq:bregman-basic}
\end{equation}
A core application of this would be the mirror descent step \citep{nem-yudin,beck-teb-mda},
where the operation is applied to a linearized version of $f$ instead
of $f$ directly:
\[
x^{k}=\arg\min_{x}\left\{ \left\langle x,\nabla f(x^{k-1})\right\rangle +\rho B_{\phi}(x,x^{k-1})\right\} .
\]
Bregman proximal operations can be interpreted as geodesic steps with
respect to the dual connection. The key idea is that given an input
point $x^{k-1}$, they output a point $x$ such that the velocity
of the connecting geodesic is equal to $-\nabla\frac{1}{\rho}f(x)$
at $x$. This velocity is measured in the flat coordinate system of
the connection, the dual coordinates. To see why, consider a geodesic
$\gamma(t)=(1-t)\nabla\phi(x^{k-1})+t\nabla\phi(x^{k}).$ Here $x^{k-1}$
and $x^{k}$ are in primal coordinates and $\gamma(t)$ is in dual
coordinates. The velocity is $\frac{d}{dt}\gamma(t)=\nabla\phi(x^{k})-\nabla\phi(x^{k-1}).$
Contrast to the optimality condition of the Bregman prox (Equation
\ref{eq:bregman-basic}):
\[
-\frac{1}{\rho}\nabla f(x^{k})=\nabla\phi(x^{k})-\nabla\phi(x^{k-1}).
\]
For instance, when using the Euclidean penalty the step is:
\[
{\textstyle x^{k}=\arg\min_{x}\bigl\{ f(x)+\frac{\rho}{2}\left\Vert x-x^{k-1}\right\Vert ^{2}\bigr\}.}
\]
The final velocity is just $x^{k}-x^{k-1}$, and so $x^{k}-x^{k-1}=-\frac{1}{\rho}\nabla f(x^{k})$,
which is the solution of the proximal operation.\vspace{-0.5em}

\section{Primal-Dual form of the proximal point method\vspace{-0.5em}
}

\label{sec:dual-form}The proximal point method is the building block
from which we will construct the accelerated gradient method. Consider
the basic form of the proximal point method applied to a strongly
convex function $f$. At each step, the iterate $x^{k}$ is constructed
from $x^{k-1}$ by solving the proximal operation subproblem given
an inverse step size parameter $\eta$:
\begin{equation}
x^{k}=\arg\min_{x}\left\{ f(x)+\frac{\eta}{2}\left\Vert x-x^{k-1}\right\Vert ^{2}\right\} .\label{eq:prox-step}
\end{equation}
This step can be considered an implicit form of the gradient step,
where the gradient is evaluated at the end-point of the step instead
of the beginning:
\[
x^{k}=x^{k-1}-\frac{1}{\eta}\nabla f(x^{k}),
\]

which is just the optimality condition of the subproblem in Equation
\ref{eq:prox-step}, found by taking the derivative $\nabla f(x)+\eta x-\eta x^{k-1}$
to be zero. A remarkable property of the proximal operation becomes
apparent when we rearrange this formula, namely that the solution
to the operation is not a single point but a \emph{primal-dual pair},
whose weighted sum is equal to the input point:
\[
x^{k}+\frac{1}{\eta}\nabla f(x^{k})=x^{k-1}.
\]

If we define $g^{k}=\nabla f(x^{k})$, the primal-dual pair obeys
a duality relation: $g^{k}=\nabla f(x^{k})$ and $x^{k}=\nabla f^{*}(g^{k})$,
which allows us to interchange primal and dual quantities freely.
Indeed we may write the condition in a dual form as:
\begin{equation}
\nabla f^{*}\left(g^{k}\right)+\frac{1}{\eta}g^{k}=x^{k-1},\label{eq:basic_dual_condition}
\end{equation}
which is the optimality condition for the proximal operation:
\[
g^{k}=\arg\min_{g}\left\{ f^{*}(g)+\frac{1}{2\eta}\left\Vert g-\eta x^{k-1}\right\Vert ^{2}\right\} .
\]

Our goal in this section is to express the proximal point method in
terms of a dual step, and while this equation involves the dual function
$f^{*}$, it is not a \emph{step} in the sense that $g^{k}$ is formed
by a proximal operation from $g^{k-1}.$ 

We can manipulate this formula further to get an update of the form
we want, by simply adding and subtracting $g^{k-1}$ from \ref{eq:basic_dual_condition}:
\[
\nabla f^{*}\left(g^{k}\right)+\frac{1}{\eta}g^{k}=\frac{1}{\eta}g^{k-1}+\left(x^{k-1}-\frac{1}{\eta}g^{k-1}\right),
\]
Which gives the updates:
\begin{align*}
g^{k} & =\arg\min_{g}\left\{ f^{*}(g)-\left\langle g,\,x^{k-1}-\frac{1}{\eta}g^{k-1}\right\rangle +\frac{1}{2\eta}\left\Vert g-g^{k-1}\right\Vert ^{2}\right\} ,\\
x^{k} & =x^{k-1}-\frac{1}{\eta}g^{k}.
\end{align*}
We call this the primal-dual form of the proximal point method.

\section{Acceleration as a change of geometry\vspace{-0.5em}
}

\label{sec:prox-change-geom}The proximal point method is rarely used
in practice due to the difficulty of computing the solution to the
proximal subproblem. It is natural then to consider modifications
of the subproblem to make it more tractable. The subproblem becomes
particularly simple if we replace the proximal operation with a Bregman
proximal operation with respect to $f^{*}$, 
\[
g^{k}=\arg\min_{g}\left\{ f^{*}(g)-\left\langle g,\,x^{k-1}-\frac{1}{\eta}g^{k-1}\right\rangle +\tau B_{f^{*}}(g,g^{k-1})\right\} .
\]
We have additionally changed the penalty parameter to a new constant
$\tau$, which is necessary as the change to the Bregman divergence
changes the scaling of distances. We discuss this further below. 

Recall from Section \ref{sec:bregman-prox} that Bregman proximal
operations follow geodesics. The key idea is that we are now following
a geodesic in the dual connection of $\phi=f^{*}$, using the notation
of Section \ref{subsec:flat}, which is a \emph{straight-line in the
primal coordinates} of $f$ due to the flatness of the connection
(Section \ref{subsec:flat}). Due to the flatness property, a simple
closed-form solution can be derived by equating the derivative to
0:
\begin{gather*}
\nabla f^{*}(g^{k})-\left[x^{k-1}-\frac{1}{\eta}g^{k-1}\right]+\tau\nabla f^{*}(g^{k})-\tau\nabla f^{*}(g^{k-1})=0,\\
\text{therefore }g^{k}=\nabla f\left(\left(1+\tau\right)^{-1}\left[x^{k-1}-\frac{1}{\eta}g^{k-1}+\tau\nabla f^{*}(g^{k-1})\right]\right).
\end{gather*}
This formula gives $g^{k}$ in terms of the derivative of known quantities,
as $\nabla f^{*}(g^{k-1})$ is known from the previous step as the
point at which we evaluated the derivative at. We will denote this
argument to the derivative operation $y$, so that $g^{k}=\nabla f(y^{k})$.
It no longer holds that $g^{k}=\nabla f(x^{k})$ after the change
of divergence. Using this relation, $y$ can be computed each step
via the update:
\[
y^{k}=\frac{x^{k-1}-\frac{1}{\eta}g^{k-1}+\tau y^{k-1}}{1+\tau}.
\]

In order to match the accelerated gradient method exactly we need
some additional flexibility in the step size used in the $y^{k}$
update. To this end we introduce an additional constant $\alpha$
in front of $g^{k-1}$, which is 1 for the proximal point variant.
The full method is as follows:

\noindent\fbox{\begin{minipage}[t]{1\columnwidth - 2\fboxsep - 2\fboxrule}%

\paragraph*{Bregman form of the accelerated gradient method}

\begin{align}
y^{k} & =\frac{x^{k-1}-\frac{\alpha}{\eta}g^{k-1}+\tau y^{k-1}}{1+\tau},\nonumber \\
g^{k} & =\nabla f(y^{k}),\nonumber \\
x^{k} & =x^{k-1}-\frac{1}{\eta}g^{k}.\label{eq:bregman-full-form}
\end{align}
\end{minipage}}

This is very close to the equational form of Nesterov's method explored
by \citet{lan2017}, with the change that they assume an explicit
regularizer is used, whereas we assume strong convexity of $f$. Indeed
we have chosen our notation so that the constants match. This form
is algebraically equivalent to other known forms of the accelerated
gradient method for appropriate choice of constants. Table \ref{tab:nes}
shows the direct relation between the many known ways of writing the
accelerated gradient method in the strongly-convex case (Proofs of
these relations are in the Appendix). When $f$ is $\mu$-strongly
convex and $L$-smooth, existing theory implies an accelerated geometric
convergence rate of at least $1-\sqrt{\frac{\mu}{L}}$ for the parameter
settings \citep{nesterov2013introductory}:
\[
{\textstyle \eta=\sqrt{\mu L},\qquad\tau=\frac{L}{\eta},\qquad\text{\ensuremath{\alpha}=\ensuremath{\frac{\tau}{1+\tau}}}}.
\]
In contrast, the primal-dual form of the proximal point method achieves
at least that convergence rate for parameters:
\[
{\textstyle \eta=\sqrt{\mu L},\qquad\tau=\frac{1}{\eta},\qquad\alpha=1}.
\]
The difference in $\tau$ arises from the difference in the scaling
of the Bregman penalty compared to the Euclidean penalty. The Bregman
generator $f^{*}$ is strongly convex with constant $1/L$ whereas
the Euclidean generator $\frac{1}{2}\left\Vert \cdot\right\Vert ^{2}$
is strongly convex with constant 1, so the change in scale requires
rescaling by $L$.

\subsection{Interpretations\vspace{-0.5em}
}

After the change in geometry, the $g$ update no longer gives a dual
point that is directly the gradient of the primal iterate. However,
notice that the term we are attempting to minimize in the $g$ step:
\[
f^{*}(g)-\bigl\langle g,\,x^{k-1}-\frac{\alpha}{\eta}g^{k-1}\bigr\rangle,
\]
has a fixed point of $\nabla f^{*}\left(g^{k}\right)=x^{k-1}-\frac{\alpha}{\eta}g^{k},$
which is precisely an $\alpha$-weight version of the proximal point's
key property from Equation \ref{eq:basic_dual_condition}. Essentially
we have relaxed the proximal-point method. Instead of this relation
holding precisely at every step, we are instead constantly taking
steps which pull $g$ closer to satisfying it.

\subsection{Inertial form\vspace{-0.5em}
}

The primal-dual view of the proximal point method can also be written
in terms of the quantity $z^{k-1}=x^{k-1}-\frac{\alpha}{\eta}g^{k-1}$
instead of $x^{k-1}$. This form is useful for the construction of
ODEs that model the discrete dynamics. Under this change of variables
the updates are:
\begin{align}
g^{k} & =\arg\min_{g}\left\{ f^{*}(g)-\left\langle g,\,z^{k-1}\right\rangle +\frac{1}{2\eta}\left\Vert g-g^{k-1}\right\Vert ^{2}\right\} ,\nonumber \\
z^{k} & =z^{k-1}-\frac{1}{\eta}g^{k}-\frac{\alpha}{\eta}\left(g^{k}-g^{k-1}\right).\label{eq:inertial-prox-point}
\end{align}

\subsection{Relation to the heavy ball method\vspace{-0.5em}
}

Consider Equation \ref{eq:bregman-full-form} with $\alpha=0$, which
removes the over-extrapolation before the proximal operation. If we
define $\beta=\frac{\tau}{1+\tau}$ we may write the method as:
\[
\begin{aligned}x^{k} & =x^{k-1}-\frac{1}{\eta}f^{\prime}(y^{k-1}),\qquad y^{k}=\beta y^{k-1}+\left(1-\beta\right)x^{k}.\end{aligned}
\]
We can eliminate $x^{k}$ from the $y^{k}$ update above by plugging
in the $x^{k}$ step equation, then using the $y^{k}$ update from
the previous step in the form $\left(1-\beta\right)x^{k-1}=y^{k-1}-\beta y^{k-2}:$
\begin{align*}
y^{k} & =\beta y^{k-1}+\left(1-\beta\right)\left(x^{k-1}-\frac{1}{\eta}f^{\prime}(y^{k-1})\right)\\
 & =\beta y^{k-1}-\left(1-\beta\right)\frac{1}{\eta}f^{\prime}(y^{k-1})+\left[y^{k-1}-\beta y^{k-2}\right]\\
 & =y^{k-1}-\left(1-\beta\right)\frac{1}{\eta}f^{\prime}(y^{k-1})+\beta\left[y^{k-1}-y^{k-2}\right].
\end{align*}
This has the exact form of the heavy ball method with step size $\left(1-\beta\right)/\eta$
and momentum $\beta$. We can also derive the heavy ball method by
starting from the saddle-point expression for $f$:
\[
\min_{x}f(x)=\min_{x}\max_{g}\left\{ \left\langle x,g\right\rangle -f^{*}(g)\right\} .
\]
The alternating-block gradient descent/ascent method on the objective
$\left\langle x,g\right\rangle -f^{*}(g)$ with step-size $\gamma$
is simply:
\[
\begin{aligned}g^{k} & =g^{k-1}+\frac{1}{\gamma}\left[x^{k-1}-\nabla f^{*}(g^{k-1})\right],\end{aligned}
\qquad x^{k}=x^{k-1}-\gamma g^{k}.
\]
If we instead perform a Bregman proximal update in the dual geometry
for the $g$ part, we arrive at the same equations as we had for the
primal-dual proximal point method but with $\alpha=0$, yielding the
heavy ball method. In order to get the accelerated gradient method
instead of the heavy ball method, the extra inertia that arises from
starting from the proximal point method instead of the saddle point
formulation appears to be crucial.

\section{Dual geometry in continuous time\vspace{-0.5em}
}

The inertial form (Equation \ref{eq:inertial-prox-point}) of the
proximal point method can be formulated as an ODE in a very natural
way, by mapping $z^{k}-z^{k-1}\rightarrow\dot{z}$ and $g^{k}-g^{k-1}\rightarrow\dot{g}$,
and taking $x$ and $g$ to be at time $t$. This is the inverse of
the Euler class of discretizations applied separately to the two terms,
which is the most natural way to discretize an ODE. The resulting
proximal point ODE is:
\[
\begin{aligned}\dot{g} & =f_{g}(z,g,t)\doteq-\frac{1}{\tau}\nabla f^{*}\left(g\right)+\frac{1}{\tau}z,\\
\dot{z} & =f_{z}(z,g,t)\doteq-\frac{1}{\eta}g-\frac{\alpha}{\eta}\dot{g}.
\end{aligned}
\]
We have suppressed the dependence on $t$ of each quantity for notational
simplicity. We can treat $g$ more formally as a point $g\in\mathcal{M}$
on a Hessian manifold $\mathcal{M}$. Then the solution for the $g$
variable of the ODE is a curve $\gamma(t):I\rightarrow\mathcal{TM}$
from an interval $I$ to the tangent bundle on the manifold so the
velocity $\dot{\gamma}(t)\in T_{g}\mathcal{M}$ obeys the ODE: $\dot{\gamma}(t)=f_{g}(z,g,t).$
The right hand side of the ODE is a point in the tangent space of
the manifold at $\gamma(t)$, expressed in Euclidean coordinates.

We can now apply the same partial change of geometry that we used
in the discrete case. We will consider the quantity $f_{g}(z,g,t)$
to be a tangent vector in dual tangent space coordinates For the $\phi=f^{*}$
Hessian manifold, instead of its primal tangent space coordinates
(which would leave the ODE unchanged). The variable $g$ remains in
primal coordinates with respect to $\phi$, so we must add to the
ODE a change of coordinates for the tangent vector, yielding:
\[
\dot{g}=\left(\nabla^{2}f^{*}(g)\right)^{-1}f_{g}(z,g,t),
\]
where we have used the inverse of Equation \ref{eq:coc-tangent},
with $\phi=f^{*}$. We can rewrite this as:
\begin{align*}
f_{g}(z,g,t)=\nabla^{2}f^{*}(g)\dot{g} & =\frac{d}{dt}\nabla f^{*}(g),
\end{align*}
giving the AGM ODE system:
\begin{gather*}
{\textstyle \frac{d}{dt}\nabla f^{*}(g)=-\frac{1}{\tau}\nabla f^{*}\left(g\right)+\frac{1}{\tau}z,\qquad\dot{z}=-\frac{1}{\eta}g-\frac{\alpha}{\eta}\dot{g}.}
\end{gather*}
It is now easily seen that the implicit Euler update for the $g$
variable with $z$ fixed now corresponds to the solution of the Bregman
proximal operation considered in the discrete case. So this ODE is
a natural continuous time analogue to the accelerated gradient method. 

\subsection*{Convergence in continuous time\vspace{-0.5em}
}

\begin{wrapfigure}{r}{0.4\columnwidth}%
\begin{centering}
\vspace{-2em}
\includegraphics[height=2in]{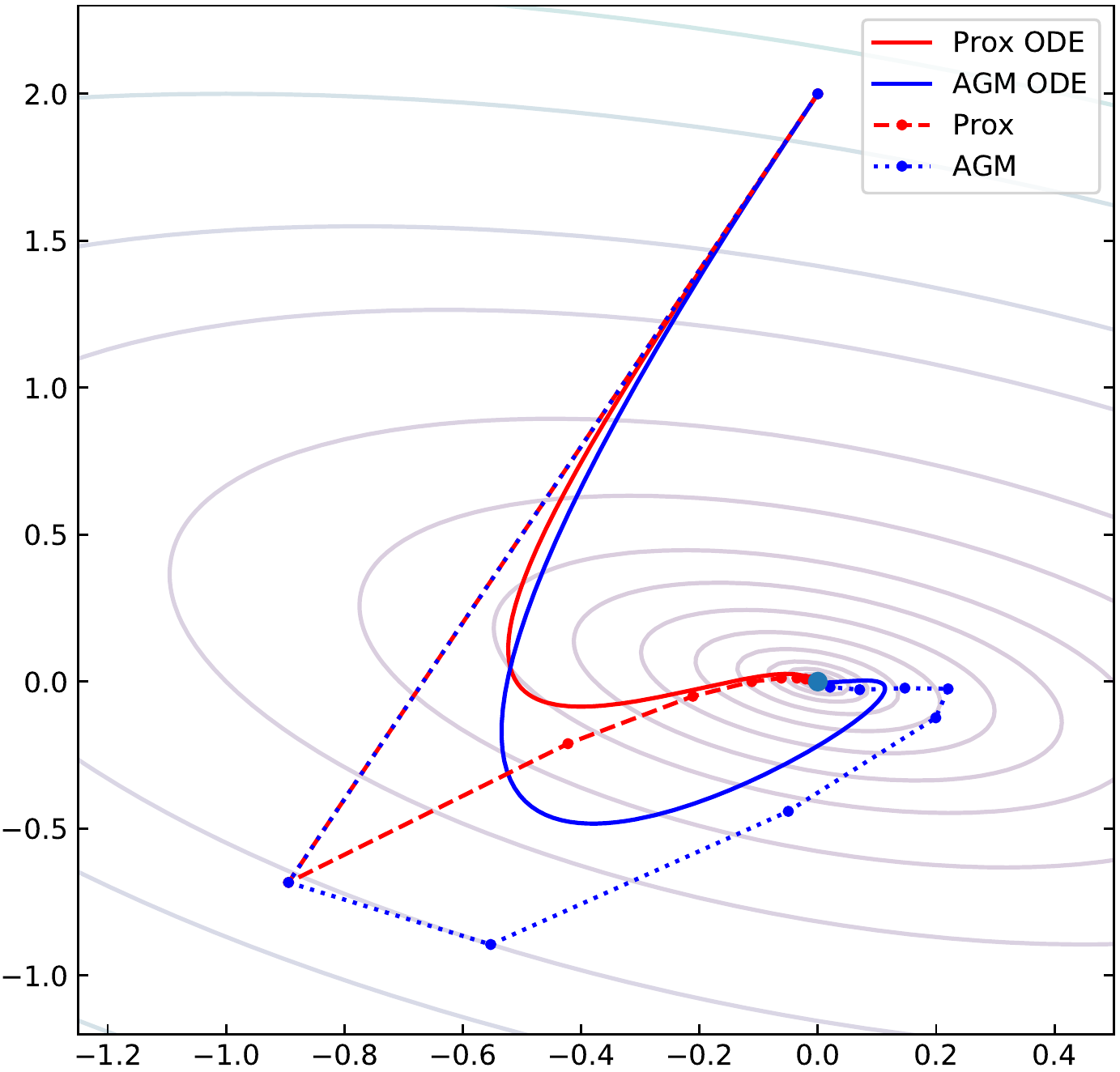}
\par\end{centering}
{\small{}\caption{{\small{}\label{fig:convergence-example}Paths for the quadratic problem
$f(x)=\frac{1}{2}x^{T}Ax$ with $A=[2,\,1;\,1,\,3].$ }}
}\vspace{0.6em}
\end{wrapfigure}%
The natural analogy to convergence in continuous time is known as
the decay rate of the ODE. A sufficient condition for an ODE with
parameters $u=[z;g]$ to decay with constant $\rho$ is:
\[
\left\Vert u(t)-u^{*}\right\Vert \leq\exp\left(-t\rho\right)\left\Vert u(0)-u^{*}\right\Vert ,
\]
where $u^{*}$ is a fixed point. We can relate this to the discrete
case by noting that $\exp(-t\rho)=\lim_{k\rightarrow\infty}(1-\frac{t}{k}\rho)^{k}$,
so given our discrete-time convergence rate is proportional to $(1-\sqrt{\mu/L})^{k},$
we would expect values of $\rho$ proportional to $\sqrt{\mu/L}$
if the ODE behaves similarly to the discrete process. We have been
able to establish this result for both the proximal and AGM ODEs for
quadratic objectives (proof in the Appendix in the supplementary material).
\begin{thm}
The proximal and AGM ODEs decay with at least the following rates
for $\mu$-strongly convex and $L$-smooth quadratic objective functions
when using the same hyper-parameters as in the discrete case:
\[
{\textstyle \rho_{\text{prox}}\geq\frac{\sqrt{\mu}}{\sqrt{\mu}+\sqrt{L}},\quad\rho_{\text{AGM}}\geq\frac{1}{2}\sqrt{\frac{\mu}{L}}}.
\]
\end{thm}
Figure \ref{fig:convergence-example} contrasts the convergence of
the discrete and continuous variants. The two methods have quite distinct
paths whose shape is shared by their ODE counterparts.

\section{Related Work\vspace{-0.5em}
}

The application of Bregman divergence to the analysis of continuous
time views of the accelerated gradient method has recently been explored
by \citet{wibisono2016variational} and \citet{wilson-jordan-2016}.
Their approaches do not use the Bregman divergence of $f^{*}$, a
key factor of our approach. The Bregman divergence of a function $\phi$
occurs explicitly as a term in a Hamiltonian, in contrast to our view
of $\phi$ as curving space. The accelerated gradient method has been
shown to be modeled by a momentum of the form ODE $\ddot{X}+c(t)\dot{X}+\nabla f(x)=0$
by \citet{candes2014}. Natural discretizations of their ODE result
in the heavy-ball method instead of the accelerated gradient method,
unlike our form which can produce both based on the choice of $\alpha$.
The fine-grained properties of momentum ODEs have also been studied
in the quadratic case by \citet{scieur_bach2017}.

A primal-dual form of the regularized accelerated gradient method
appears in \citet{lan2017}. Our form can be seen as a special case
of their form when the regularizer is zero. Our work extends theirs,
providing an understanding of the role that geometry plays in unifying
acceleration and implicit steps.

The Riemannian connection induced by a function has been heavily explored
in the optimization literature as part of the natural gradient method
\citep{amari1998natural}, although other connections on this manifold
are less explored. The dual-flat connections have primarily seen use
in the information-geometry setting for optimization over distributions
\citep{infogeombook}.

The accelerated gradient method is not the only way to achieve accelerated
rates among first order methods. Other techniques include the Geometric
descent method of \citet{bubeck2015geometric}, where a bounding ball
is updated at each step that encloses two other balls, a very different
approach. The method described by \citet{nem-yudin} is also notable
as being closer to the conjugate gradient method than other accelerated
approaches, but at the expense of requiring a 2D search instead of
a 1D line search at each step.

\section{Conclusion}

We believe the tools of differential geometry may provide a new and
insightful avenue for the analysis of accelerated optimization methods.
The analysis we provide in this work is a first step in this direction.
The advantage of the differential geometric approach is that it provides
high level tools that make the derivation of acceleration easier to
state. This derivation, from the proximal point method to the accelerated
gradient method, is in our opinion not nearly as mysterious as the
other known approaches to understanding acceleration.

\bibliographystyle{plainnat}
\nocite{*}
\bibliography{curved_geometry}

\newpage{}

\appendix

\part*{Appendix}

\section{Reformulations of the accelerated gradient method}

\subsection*{Form II}

This simplification is described in \citet{nesterov2013introductory}
which we reproduce for completeness. Recall that Form I is given by
the updates:
\[
\begin{aligned}y^{k} & =\frac{\alpha\gamma v^{k}+\gamma x^{k}}{\alpha\mu+\gamma},\\
x^{k+1} & =y^{k}-\frac{1}{L}\nabla f(y^{k}),\\
v^{k+1} & =\left(1-\alpha\right)v^{k}+\frac{\alpha\mu}{\gamma}y^{k}-\frac{\alpha}{\gamma}\nabla f(y^{k}).
\end{aligned}
\]
Nesterov specifies the requirement that $\gamma=(1-\alpha)\gamma+\alpha\mu$.
When $\alpha=\sqrt{\mu/L}$ then $\gamma$ must then satisfy:
\[
\gamma=(1-\sqrt{\mu/L})\gamma+\mu\sqrt{\mu/L},
\]
\[
\therefore1=(1-\sqrt{\mu/L})+\frac{\mu}{\gamma}\sqrt{\mu/L},
\]
\[
\therefore\sqrt{\mu/L})=\frac{\mu}{\gamma}\sqrt{\mu/L},
\]
\[
\therefore\mu=\gamma
\]
We may rewrite the $y^{k}$ definition as:
\[
\left(\alpha\mu+\gamma\right)y^{k}=\alpha\gamma v^{k}+\gamma x^{k},
\]
\[
\therefore v^{k}=\frac{1}{\alpha\gamma}\left[\left(\alpha\mu+\gamma\right)y^{k}-\gamma x^{k}\right].
\]
Plugging this into the $v$ step:
\begin{align*}
v_{k+1} & =\left(1-\alpha\right)v^{k}+\frac{\alpha\mu}{\gamma}y^{k}-\frac{\alpha}{\gamma}\nabla f(y^{k})\\
 & =\frac{1-\alpha}{\alpha\gamma}\left[\left(\alpha\mu+\gamma\right)y^{k}-\gamma x^{k}\right]+\frac{\alpha\mu}{\gamma}y^{k}-\frac{\alpha}{\gamma}\nabla f(y^{k})\\
 & =\frac{1}{\gamma}\left[\left(1-\alpha\right)\mu y^{k}+\left(\frac{1-\alpha}{\alpha}\right)\gamma y^{k}+\alpha\mu y^{k}\right]-\frac{1-\alpha}{\alpha}x^{k}-\frac{\alpha}{\gamma}\nabla f(y^{k})\\
 & =\frac{1}{\alpha\gamma}\left[\alpha\mu y^{k}+\left(1-\alpha\right)\gamma y^{k}\right]-\frac{1-\alpha}{\alpha}x^{k}-\frac{\alpha}{\gamma}\nabla f(y^{k})\\
 & =\frac{1}{\alpha}\left[y^{k}\right]-\frac{1-\alpha}{\alpha}x^{k}-\frac{\alpha}{\gamma}\nabla f(y^{k})\\
 & =x^{k}+\frac{1}{\alpha}\left(y^{k}-x^{k}\right)-\frac{\alpha}{\gamma}\nabla f(y^{k})\\
 & =x^{k}+\frac{1}{\alpha}\left(x^{k+1}-x^{k}\right).
\end{align*}

For $y^{k}$, we start with using $\gamma=\mu$, then apply the $v$
simplification;
\begin{align*}
y^{k} & =\frac{\alpha\gamma v^{k}+\gamma x^{k}}{\alpha\gamma+\gamma}\\
 & =\frac{\left(\alpha\gamma x^{k-1}+\gamma\left(x^{k}-x^{k-1}\right)\right)+\gamma x^{k}}{\alpha\gamma+\gamma}\\
 & =x^{k}+\frac{\left(\alpha\gamma x^{k-1}+\gamma\left(x^{k}-x^{k-1}\right)\right)-\alpha\gamma x^{k}}{\alpha\gamma+\gamma}\\
 & =x^{k}+\frac{\alpha\gamma\left(x^{k-1}-x^{k}\right)+\gamma\left(x^{k}-x^{k-1}\right)}{\alpha\gamma+\gamma}\\
 & =x^{k}+\frac{\gamma-\alpha\gamma}{\alpha\gamma+\gamma}\left(x^{k}-x^{k-1}\right).
\end{align*}
Note that by multiplying by $\sqrt{L}/\mu$ we get:
\[
\frac{\mu-\mu\sqrt{\mu/L}}{\mu\sqrt{\mu/L}+\mu}=\frac{\sqrt{L}-\sqrt{\mu}}{\sqrt{\mu}+\sqrt{L}}=\beta.
\]

\subsection*{Sutskever's form}

Recall Sutskever's form:
\[
\begin{aligned}p_{\text{Sut}}^{k+1} & =\beta p_{\text{Sut}}^{k}-\frac{1}{L}\nabla f\left(x_{\text{Sut}}^{k}+\beta p_{\text{Sut}}^{k}\right),\\
x_{\text{Sut}}^{k+1} & =x_{\text{Sut}}^{k}+p_{\text{Sut}}^{k+1}.
\end{aligned}
\]
and Nesterov's form:
\[
\begin{aligned}x_{\text{Nes}}^{k+1} & =y_{\text{Nes}}^{k}-\frac{1}{L}\nabla f(y_{\text{Nes}}^{k}),\\
y_{\text{Nes}}^{k+1} & =x_{\text{Nes}}^{k+1}+\beta\left(x_{\text{Nes}}^{k+1}-x_{\text{Nes}}^{k}\right).
\end{aligned}
\]
We will show that using the substitutions:
\[
p_{\text{Sut}}^{k+1}=x_{\text{Nes}}^{k+1}-x_{\text{Nes}}^{k},
\]
\[
y_{\text{Nes}}^{k}=x_{\text{Sut}}^{k}+\beta p_{\text{Sut}}^{k},
\]
applied to Sutskever's form gives Nesterov's form. We start with the
momentum term:
\[
p_{\text{Sut}}^{k+1}=\beta p_{\text{Sut}}^{k}-\frac{1}{L}\nabla f\left(y_{\text{Nes}}^{k}\right),
\]
\[
\therefore x_{\text{Nes}}^{k+1}-x_{\text{Nes}}^{k}=\beta\left(x_{\text{Nes}}^{k}-x_{\text{Nes}}^{k-1}\right)-\frac{1}{L}\nabla f\left(y_{\text{Nes}}^{k}\right),
\]
\[
\therefore x_{\text{Nes}}^{k+1}+\frac{1}{L}\nabla f\left(y_{\text{Nes}}^{k}\right)=x_{\text{Nes}}^{k}+\beta\left(x_{\text{Nes}}^{k}-x_{\text{Nes}}^{k-1}\right).
\]
Defining: $y_{\text{Nes}}^{k}:=x_{\text{Nes}}^{k+1}+\frac{1}{L}\nabla f\left(y_{\text{Nes}}^{k}\right)$
and applying on the right gives Nesterov's $y$ update:
\[
y_{\text{Nes}}^{k}=x_{\text{Nes}}^{k}+\beta\left(x_{\text{Nes}}^{k}-x_{\text{Nes}}^{k-1}\right).
\]

\subsection*{Modern form}

We want:
\[
\begin{aligned}p_{\text{Mod}}^{k+1} & =\beta p_{\text{Mod}}^{k}+\nabla f(x_{\text{Mod}}^{k}),\\
x_{\text{Mod}}^{k+1} & =x_{\text{Mod}}^{k}-\frac{1}{L}\left(\nabla f(x_{\text{Mod}}^{k})+\beta p_{\text{Mod}}^{k+1}\right).
\end{aligned}
\]
Starting from Sutskever's form, 
\[
\begin{aligned}p_{\text{Sut}}^{k+1} & =\beta p_{\text{Sut}}^{k}-\frac{1}{L}\nabla f\left(x_{\text{Sut}}^{k}+\beta p_{\text{Sut}}^{k}\right),\\
x_{\text{Sut}}^{k+1} & =x_{\text{Sut}}^{k}+p_{\text{Sut}}^{k+1}.
\end{aligned}
\]
Define $x_{\text{Mod}}^{k}=x_{\text{Sut}}^{k}+\beta p_{\text{Sut}}^{k}$.
Note that this is equal to $y_{\text{Nes}}^{k}$ by definition. So
we have:
\[
x_{\text{Sut}}^{k}=x_{\text{Mod}}^{k}-\beta p_{\text{Sut}}^{k}.
\]
Plugging that into the Sutskever step
\[
x_{\text{Mod}}^{k+1}-\beta p_{\text{Sut}}^{k+1}=x_{\text{Mod}}^{k}-\beta p_{\text{Sut}}^{k}+p_{\text{Sut}}^{k+1},
\]
\begin{align*}
\therefore x_{\text{Mod}}^{k+1} & =x_{\text{Mod}}^{k}+\left(p_{\text{Sut}}^{k+1}-\beta p_{\text{Sut}}^{k}\right)+\beta p_{\text{Sut}}^{k+1}\\
 & =x_{\text{Mod}}^{k}-\frac{1}{L}f^{\prime}\left(x_{\text{Mod}}^{k}\right)+\beta p_{\text{Sut}}^{k+1}.
\end{align*}

Then define $p_{\text{mod}}^{k}=-Lp_{\text{Sut}}^{k},$ so the update
becomes:
\[
x_{\text{Mod}}^{k+1}=x_{\text{Mod}}^{k}-\frac{1}{L}\left(\nabla f\left(x_{\text{Mod}}^{k}\right)+\beta p_{\text{Sut}}^{k+1}\right).
\]
The momentum update changes from:
\[
p_{\text{Sut}}^{k+1}=\beta p_{\text{Sut}}^{k}-\frac{1}{L}\nabla f\left(x_{\text{Sut}}^{k}+\beta p_{\text{Sut}}^{k}\right),
\]
to:
\[
-\frac{1}{L}p_{\text{mod}}^{k+1}=-\beta\frac{1}{L}p_{\text{mod}}^{k}-\frac{1}{L}\nabla f\left(x_{\text{Mod}}^{k}\right),
\]
\[
\therefore p_{\text{mod}}^{k+1}=\beta p_{\text{mod}}^{k}+\nabla f\left(x_{\text{Mod}}^{k}\right).
\]

\subsection*{Auslender \& Teboulle form }

\paragraph*{
\[
y_{\text{AT}}^{k}=(1-\theta_{\text{AT}})\hat{x}_{\text{AT}}^{k}+\theta_{\text{AT}}z_{\text{AT}}^{k},
\]
\[
z_{\text{AT}}^{k+1}=z_{\text{AT}}^{k}-\frac{\gamma_{\text{AT}}}{\theta_{\text{AT}}}\nabla f(y_{\text{AT}}^{k}),
\]
\[
\hat{x}_{\text{AT}}^{k+1}=(1-\theta_{\text{AT}})\hat{x}_{\text{AT}}^{k}+\theta_{\text{AT}}z_{\text{AT}}^{k+1}.
\]
}

We first eliminate $y_{\text{AT}}^{k}$ from the $\hat{x}_{\text{AT}}^{k+1}$
update:
\begin{align*}
\hat{x}_{\text{AT}}^{k+1} & =(1-\theta_{\text{AT}})\hat{x}_{\text{AT}}^{k}+\theta_{\text{AT}}z_{\text{AT}}^{k+1}\\
 & =(1-\theta_{\text{AT}})\hat{x}_{\text{AT}}^{k}+\theta_{\text{AT}}\left(z_{\text{AT}}^{k}-\frac{\gamma_{\text{AT}}}{\theta_{\text{AT}}}\nabla f(y_{\text{AT}}^{k})\right)\\
 & =\hat{x}_{k}+\theta_{\text{AT}}\left(z_{\text{AT}}^{k}-\hat{x}_{\text{AT}}^{k}\right)-\gamma_{\text{AT}}\nabla f\left(\hat{x}_{k}+\theta_{\text{AT}}\left(z_{\text{AT}}^{k}-\hat{x}_{\text{AT}}^{k}\right)\right).
\end{align*}
This notational similarity with Nesterov's $x_{\text{Nes}}^{k+1}=y_{\text{Nes}}^{k}-\frac{1}{L}\nabla f(y_{\text{Nes}}^{k})$
suggests matching 
\[
y_{\text{Nes}}^{k}=x_{\text{mod}}^{k}=y_{\text{AT}}^{k}=\hat{x}_{k}+\theta_{\text{AT}}\left(z_{\text{AT}}^{k}-\hat{x}_{\text{AT}}^{k}\right),
\]
as well as $\gamma_{\text{AT}}=\frac{1}{L}$ and
\[
x_{\text{Nes}}^{k}=\hat{x}_{\text{AT}}^{k}.
\]
Note that using using the step for $\hat{x}$ we can rearrange to
get:
\[
z_{\text{AT}}^{k+1}=\frac{1}{\theta_{\text{AT}}}\hat{x}_{\text{AT}}^{k+1}-\frac{(1-\theta_{\text{AT}})}{\theta_{\text{AT}}}\hat{x}_{\text{AT}}^{k}.
\]

Now to determine $\theta$ we simplify using this substitution:
\begin{align*}
y_{\text{AT}}^{k} & =\hat{x}_{\text{AT}}^{k}+\theta_{\text{AT}}\left(z_{\text{AT}}^{k}-\hat{x}_{\text{AT}}^{k}\right)\\
 & =\hat{x}_{\text{AT}}^{k}+\theta_{\text{AT}}\left(\frac{1}{\theta_{\text{AT}}}\hat{x}_{\text{AT}}^{k}-\frac{(1-\theta_{\text{AT}})}{\theta_{\text{AT}}}\hat{x}_{\text{AT}}^{k-1}\right)\\
 & =\hat{x}_{\text{AT}}^{k}+\theta_{\text{AT}}\left(\frac{1-\theta_{\text{AT}}}{\theta_{\text{AT}}}\hat{x}_{\text{AT}}^{k}-\frac{(1-\theta_{\text{AT}})}{\theta_{\text{AT}}}\hat{x}_{\text{AT}}^{k-1}\right)\\
 & =\hat{x}_{\text{AT}}^{k}+\left(1-\theta_{\text{AT}}\right)\left(\hat{x}_{\text{AT}}^{k}-\hat{x}_{\text{AT}}^{k-1}\right).
\end{align*}
So my matching constants against Nesterov's method we have $\beta_{\text{Nes}}=1-\theta_{\text{AT}}.$

\subsection*{Lan form}

\[
\begin{aligned}\tilde{x}_{\text{Lan}}^{k} & =\alpha_{\text{Lan}}(x_{\text{Lan}}^{k-1}-x_{\text{Lan}}^{k-2})+x_{\text{Lan}}^{k-1},\\
\underline{x}_{\text{Lan}}^{k} & =\frac{\tilde{x}_{\text{Lan}}^{k}+\tau_{\text{Lan}}\underline{x}_{\text{Lan}}^{k-1}}{1+\tau_{\text{Lan}}},\\
g_{\text{Lan}}^{k} & =\nabla f(\underline{x}_{\text{Lan}}^{k}),\\
x_{\text{Lan}}^{k} & =x_{\text{Lan}}^{k-1}-\frac{1}{\eta_{\text{Lan}}}g_{\text{Lan}}^{k}.
\end{aligned}
\]
Now we can eliminate $\tilde{x}$ from the updates to give
\[
\underline{x}_{\text{Lan}}^{k}=\frac{\tau_{\text{Lan}}\underline{x}_{\text{Lan}}^{k-1}+x_{\text{Lan}}^{k-1}+\alpha_{\text{Lan}}(x_{\text{Lan}}^{k-1}-x_{\text{Lan}}^{k-2})}{1+\tau_{\text{Lan}}},
\]
Also consider the $y^{k}$ update for the AT method:

\paragraph*{
\[
y_{\text{AT}}^{k}=(1-\theta_{\text{AT}})\hat{x}_{\text{AT}}^{k}+\theta_{\text{AT}}z_{\text{AT}}^{k},
\]
}

We can write this to not involve $\hat{x}$, giving:
\[
y_{\text{AT}}^{k}=(1-\theta_{\text{AT}})y_{\text{AT}}^{k-1}+\theta_{\text{AT}}z_{AT}^{k}+(1-\theta_{\text{AT}})\theta_{\text{AT}}\left(z_{\text{AT}}^{k}-z_{\text{AT}}^{k-1}\right).
\]
This form suggests matching the iterates with:
\[
\underline{x}_{\text{Lan}}^{k}=y_{\text{AT}}^{k},\quad x_{\text{Lan}}^{k}=z_{\text{AT}}^{k}.
\]
Under this matching, the constants need to satisfy the following relations:
\[
\left(1-\theta_{\text{AT}}\right)=\frac{\tau_{\text{Lan}}}{1+\tau_{\text{Lan}}},
\]
\[
\left(\theta_{\text{AT}}+\theta_{\text{AT}}-\theta_{\text{AT}}^{2}\right)=\frac{1+\alpha_{\text{Lan}}}{1+\tau_{\text{Lan}}},
\]
\[
\left(1-\theta_{\text{AT}}\right)\theta_{\text{AT}}=\frac{\alpha_{\text{Lan}}}{1+\tau_{\text{Lan}}}.
\]

The settings:
\[
\tau_{\text{Lan}}=\frac{1-\theta_{\text{AT}}}{\theta_{\text{AT}}},\;\alpha_{\text{Lan}}=1-\theta_{\text{AT}},
\]
result in these three constraints being satisfied:

\[
\frac{\tau}{1+\tau}=\frac{\frac{1-\theta}{\theta}}{1+\frac{1-\theta}{\theta}}=\frac{1-\theta}{\theta+1-\theta}=1-\theta\checkmark,
\]
\[
\frac{1+\alpha}{1+\tau}=\frac{1+1-\theta}{1+\frac{1-\theta}{\theta}}=\frac{2\theta-\theta^{2}}{\theta+1-\theta}=2\theta-\theta^{2}\checkmark,
\]
\[
\frac{\alpha}{1+\tau}=\frac{1-\theta}{1+\frac{1-\theta}{\theta}}=\frac{\theta-\theta^{2}}{1}\checkmark.
\]
Matching the $x_{\text{Lan}}^{k}$ step $x_{\text{Lan}}^{k}=x_{\text{Lan}}^{k-1}-\frac{1}{\eta_{\text{Lan}}}g_{\text{Lan}}^{k}.$
requires $\frac{1}{\eta_{\text{Lan}}}=\frac{\gamma_{\text{AT}}}{\theta_{\text{AT}}}$
also.

\section{Continuous time theory}
\begin{lem}
\label{lem:ode-exp}Consider a linear ODE of the form:
\[
\dot{u}=A\left(u-u^{*}\right),
\]
Where $A$ is real and diagonalizable but not necessarily normal or
symmetric, and $x^{*}$ is the fixed point. Then
\begin{align*}
\left\Vert u(t)-u^{*}\right\Vert  & \leq\text{\ensuremath{\max}}_{j}\exp\left(tRe[\lambda_{j}]\right)\left\Vert u(0)-u^{*}\right\Vert ,
\end{align*}
where $\lambda_{j}$ are eigenvalues of $A$ with real part $Re[\lambda_{j}]$.
\end{lem}
\begin{proof}
Because the inhomogeneous term $Au^{*}$ is constant w.r.t time, without
loss of generality we can reduce our problem to a homogenous ODE by
shifting the origin:
\[
\dot{u}=Au.
\]
The solution of a linear ODE of this form is given in closed form
using the matrix exponential:
\[
u(t)=\exp\left(tA\right)u(0),
\]
we can use this formula to bound the norm of $u(t)$:
\[
\therefore\left\Vert u(t)\right\Vert \leq\left\Vert \exp\left(tA\right)\right\Vert \left\Vert u(0)\right\Vert ,
\]
where $\left\Vert \cdot\right\Vert $ is the spectral norm. If $A$
is diagonalizable with $A=U\text{diag}(\lambda_{1},\dots\lambda,d)U^{-1}$
then the matrix exponential can be expressed as:
\[
\exp\left(tA\right)=U\left[\begin{array}{ccc}
\exp(t\lambda_{1}) & 0 & 0\\
0 & \ddots & 0\\
0 & 0 & \exp(t\lambda_{d})
\end{array}\right]U^{-1}.
\]

The spectral norm is given by the largest absolute value of the eigenvalues,
so we must consider the interaction of the real and complex parts.
For an eigenvalue $\lambda=a+bi$ of $A$, the norm takes the simple
form:
\begin{align*}
\left|\exp(t\lambda)\right| & =\text{\ensuremath{\exp}}\left(ta+tbi\right)\\
 & =\left|\text{\ensuremath{\exp}}\left(ta\right)\right|\left|\text{\ensuremath{\exp}}\left(tbi\right)\right|\\
 & =\left|\text{\ensuremath{\exp}}\left(ta\right)\right|.
\end{align*}
So the spectral norm is given by the maximum over the eigenvalues
$\lambda_{j}$ of $\exp\left(tRe[\lambda_{j}]\right)$.
\end{proof}
\begin{thm}
\label{thm:prox-ode-linear}Consider the following linear ODE:
\[
\dot{u}=A\left(u-u^{*}\right),
\]
\[
A:2n\times2n=\left[\begin{array}{cc}
-I & \;-\frac{1}{\eta}I+H^{-1}\\
\eta I & -\eta H^{-1}
\end{array}\right],
\]
where $H:n\times n$ is a real, positive definite and symmetric matrix
with minimum eigenvalue $\mu$ and maximum eigenvalue $L$. This corresponds
to the proximal ODE for a quadratic function $f(x)=\frac{1}{2}\left(x-x^{*}\right)^{T}H\left(x-x^{*}\right)$,
with $u=[x;g]$ and $u^{*}=[x^{*};0]$. Then the decay rate of the
ODE towards the origin $u^{*}=0$ can be bounded as follows for $\eta=\sqrt{\mu L}$:

\[
\left\Vert u(t)-u^{*}\right\Vert \leq\exp\left(-t\rho\right)\left\Vert u(0)-u^{*}\right\Vert ,
\]
\[
\rho=\frac{\sqrt{\mu}}{\sqrt{\mu}+\sqrt{L}}.
\]
\end{thm}
\begin{proof}
We will take the approach of bounding the real parts of the eigenvalues
of $A$, so that we can directly apply Lemma \ref{lem:ode-exp}. 

Let $U\Lambda U^{T}=H$ be the eigen-decomposition of $H$. Note that
the operation of conjugation by $U$ leaves the identity matrix unchanged
($UIU^{T}=I$). Each block of A is just a weighted combination of
the identity matrix and $H^{-1}$, so this implies that conjugation
of a block by it's self gives a diagonal matrix. We can use this idea
to define a similarity transform that converts $A$ into a matrix
where each of the four blocks are diagonal. In particular we have:
\[
\left[\begin{array}{cc}
U & 0\\
0 & U
\end{array}\right]^{T}\left[\begin{array}{cc}
A_{11} & A_{12}\\
A_{21} & A_{22}
\end{array}\right]\left[\begin{array}{cc}
U & 0\\
0 & U
\end{array}\right]=\left[\begin{array}{cc}
UA_{11}U^{T} & UA_{12}U^{T}\\
UA_{21}U^{T} & UA_{22}U^{T}
\end{array}\right],
\]
where we have written A in terms of its 4 constituent $n\times n$
blocks. Each block is a weighted sum of the identity matrix and the
diagonal matrix of inverse eigenvalues of H, for instance:
\[
UA_{12}U^{T}=-\frac{1}{\eta}I+\Lambda^{-1}.
\]

Next we construct a permutation matrix $\Pi:2n\times2n$ with the
goal of converting $D$ into a block diagonal matrix with $2\times2$
blocks along the diagonal, where each block has the structure of $A$
as if it was applied to a 1D optimization problem, with $H$ being
replaced by one of the $d$ eigenvalues of $H$. This is achieved
with the permutation matrix $\Pi$ that is zero except for:
\[
\Pi_{2i,i}=1,\;\Pi_{2i+1,d+i}=1,\;i=1\dots n.
\]
 For instance, in the $n=2$ case the matrix is:
\[
\left[\begin{array}{cccc}
1 & 0 & 0 & 0\\
0 & 0 & 1 & 0\\
0 & 1 & 0 & 0\\
0 & 0 & 0 & 1
\end{array}\right].
\]
This matrix has the effect of interleaving the primal dual pairs (per
coordinate) instead of having all the primal coordinates together
followed by all the dual coordinates. So when we conjugate using $\Pi$
we get:
\[
\Pi^{T}\left[\begin{array}{cc}
U & 0\\
0 & U
\end{array}\right]^{T}A\left[\begin{array}{cc}
U & 0\\
0 & U
\end{array}\right]\Pi=\left[\begin{array}{ccc}
T_{1} & 0 & \dots\\
0 & \ddots & 0\\
\vdots & 0 & T_{d}
\end{array}\right],
\]
where each $T$ is a $2\times2$ matrix of the described form:
\[
T_{i}=\left[\begin{array}{cc}
-1 & \;-\frac{1}{\eta}+\lambda_{i}^{-1}\\
\eta & -\eta\lambda_{i}^{-1}
\end{array}\right].
\]
 The eigenvalues of a block diagonal matrix are just the eigenvalues
of the blocks concatenated, and since there is a similarity transform
between $A$ and this block diagonal matrix, we have effectively reduced
our problem to considering 1D quadratics, with curvature between $\mu$
and $L$, for fixed $\eta$. 

Recall that for a matrix $\left[\begin{array}{cc}
a & b\\
c & d
\end{array}\right]$ the eigenvalues are given by the two roots of a quadratic, namely:
\[
\nu_{\pm}=\frac{a+d\pm\sqrt{(a+d)^{2}-4(ad-bc)}}{2}.
\]
We use the notation $\nu$ to avoid confusion between the eigenvalues
of the $T$ blocks and those of H . For a block $T_{i}$, this expression
is
\[
\nu_{+}^{(i)}=-\frac{1}{2}-\frac{\eta\lambda_{i}^{-1}}{2}\pm\frac{1}{2}\sqrt{\left(1+\eta\lambda_{i}^{-1}\right)^{2}-4\left(\eta\lambda_{i}^{-1}+1-\eta\lambda_{i}^{-1}\right)}
\]
Suppose that the discriminate (the quantity under the square root)
is negative, then 
\[
Re\left[\nu_{\pm}^{(i)}\right]=-\frac{1}{2}-\frac{\eta\lambda_{i}^{-1}}{2},
\]
this is obviously at least as small as $-\rho=-\frac{\sqrt{\mu}}{\sqrt{\mu}+\sqrt{L}},$
since the largest value of $\rho$ possible is when $\mu=L$, in which
case $\rho=\frac{1}{2}$. So consider instead the case where the discriminate
is positive. We need only consider the $v_{+}$ root as it is strictly
larger. Then we will use the concavity of the square root function
to bound $\nu_{+}$, 
\[
h(x)\leq h(y)+\left\langle \nabla h(y),x-y\right\rangle 
\]
For $h=\sqrt{\cdot},$ with $y=\left(-1-\eta\lambda_{i}^{-1}\right)^{2}$
and $x=\left(-1-\eta\lambda_{i}^{-1}\right)^{2}-4$. We get:
\begin{align*}
\nu_{+}^{(i)} & \leq-\frac{1}{2}-\frac{\eta\lambda_{i}^{-1}}{2}+\frac{1}{2}\sqrt{\left(1+\eta\lambda_{i}^{-1}\right)^{2}}-\frac{4}{4\left(1+\eta\lambda_{i}^{-1}\right)}\\
 & =-\frac{1}{\left(1+\eta\lambda_{i}^{-1}\right)}.
\end{align*}
Therefore:
\[
\nu_{+}^{(i)}\leq-\frac{1}{\left(1+\eta\lambda_{i}^{-1}\right)}\leq-\frac{1}{\left(1+\eta/\mu\right)}=-\frac{\sqrt{\mu}}{\sqrt{\mu}+\sqrt{L}}=-\rho.
\]
Thus we have shown that the real parts of all eigenvalues of $A$
are less than $-\rho$.
\end{proof}
\begin{thm}
Consider the following linear ODE:
\[
\dot{u}=A\left(u-u^{*}\right),
\]
\[
A:2d\times2d=\left[\begin{array}{cc}
-\frac{\alpha}{\eta\tau}H & -\frac{1}{\eta}I+\frac{\alpha}{\eta\tau}I\\
\frac{1}{\tau}H & -\frac{1}{\tau}I
\end{array}\right],
\]

where $H:d\times d$ is a real, positive definite and symmetric matrix
with minimum eigenvalue $\mu$ and maximum eigenvalue $L$. This corresponds
to the AGM ODE for a quadratic objective$f(x)=\frac{1}{2}\left(x-x^{*}\right)^{T}H\left(x-x^{*}\right)$,
with $u=[x;g]$ and $u^{*}=[x^{*};0]$. The decay rate of this ODE
towards the origin $u^{*}=0$ can be bounded as follows for $\eta=\sqrt{\mu L}$,
$\tau=L/\eta$, and $\alpha\in[0,1]$:

\[
\left\Vert u(t)\right\Vert \leq\exp\left(-t\rho\right)\left\Vert u(0)\right\Vert ,
\]
\[
\rho=\frac{1}{2}\sqrt{\frac{\mu}{L}}.
\]
\end{thm}
\begin{proof}
We can apply the same proof technique as for Theorem \ref{thm:prox-ode-linear},
we omit the details. The $2\times2$ block diagonal matrices are:
\[
T_{i}=\left[\begin{array}{cc}
-\frac{\alpha}{\eta\tau}\lambda_{i} & -\frac{1}{\eta}+\frac{\alpha}{\eta\tau}\\
\frac{1}{\tau}\lambda_{i} & -\frac{1}{\tau}
\end{array}\right].
\]

For eigenvalues $\lambda_{i}$ of $H$. The two eigenvalues of $T_{i}$
are:
\[
\nu_{\pm}^{(i)}=-\frac{\alpha}{2\eta\tau}\lambda_{i}-\frac{1}{2\tau}\pm\frac{1}{2}\sqrt{\left(\frac{\alpha}{\eta\tau}\lambda_{i}+\frac{1}{\tau}\right)^{2}-4\frac{\lambda_{i}}{\eta\tau}}
\]
The discriminate here is always non-positive. To see why, we plug
in the constants $\eta$ and $\tau$:
\begin{align*}
\left(\frac{\alpha}{L}\lambda_{i}+\sqrt{\frac{\mu}{L}}\right)^{2}-4\frac{\lambda_{i}}{L} & \leq2\frac{\alpha^{2}}{L^{2}}\lambda_{i}^{2}+2\frac{\mu}{L}-4\frac{\lambda_{i}}{L}\\
 & \leq2\frac{\lambda_{i}}{L}+2\frac{\mu}{L}-4\frac{\lambda_{i}}{L}\\
 & =2\frac{\mu}{L}-2\frac{\lambda_{i}}{L}\\
 & \leq0.
\end{align*}
So the real part is given by the quantity outside the square root,
namely:
\[
Re\left[\nu_{\pm}^{(i)}\right]=-\frac{\alpha}{2\eta\tau}\lambda_{i}-\frac{1}{2\tau}=-\frac{\alpha\lambda_{i}}{2L}-\frac{1}{2}\sqrt{\frac{\mu}{L}}\leq-\frac{1}{2}\sqrt{\frac{\mu}{L}}.
\]
Using Lemma \ref{lem:ode-exp} gives the result.
\end{proof}

\section{The standard heavy ball ODE}

The standard heavy-ball ODE for a quadratic $f(x)=\frac{1}{2}x^{T}Hx$
is
\[
\ddot{x}+\left(1-\beta\right)\dot{x}+\gamma Hx=0.
\]
Which can be written in first-order form in terms of a momentum parameter
$p$ as:
\begin{align}
\dot{x} & =p,\nonumber \\
\dot{p} & =-\left(1-\beta\right)p-\gamma Hx.\label{eq:heavy-ball-ode}
\end{align}
The constants that result in optimal convergence rate for the discretize
heavy ball method:
\begin{align*}
x^{k+1} & =x^{k}+p^{k}\\
p^{k} & =\beta p^{k-1}-\gamma\nabla f(x^{k})
\end{align*}
which can also be written as:
\[
x^{k+1}=x^{k}-\gamma Hx^{k}+\beta\left(x^{k}-x^{k-1}\right),
\]
are:
\begin{equation}
\beta=\frac{\sqrt{L}-\sqrt{\mu}}{\sqrt{L}+\sqrt{\mu}},\quad\gamma=\frac{4}{\left(\sqrt{L}+\sqrt{\mu}\right)^{2}}.\label{eq:heavy-ball-consts}
\end{equation}

\begin{thm}
Consider the following ODE:
\[
\dot{u}=A\left(u-u^{*}\right),
\]
\[
A=\left[\begin{array}{cc}
0 & I\\
-\gamma H & -\left(1-\beta\right)I
\end{array}\right].
\]
This is the ODE in Equation \ref{eq:heavy-ball-ode} written in matrix
form for a combined iterate $u$. For the parameters given in Equation
\ref{eq:heavy-ball-consts}, this ODE has decay rate at least:
\[
\left\Vert u(t)-u^{*}\right\Vert \leq\exp\left(-t\rho\right)\left\Vert u(0)-u^{*}\right\Vert ,
\]
\[
\text{where\, }\rho=\frac{\sqrt{\mu}}{\sqrt{L}+\sqrt{\mu}}.
\]
\end{thm}
\begin{proof}
We can reduce the problem to considering the eigenvalues of $2\times2$
matrices as we did for the proximal and AGM ODEs. We have matrices
of the form:
\[
T_{i}=\left[\begin{array}{cc}
0 & 1\\
-\gamma\lambda_{i} & -\left(1-\beta\right)
\end{array}\right],
\]
whose eigenvalues are given by the general formula
\[
\nu_{\pm}=\frac{a+d\pm\sqrt{(a+d)^{2}-4(ad-bc)}}{2}.
\]
Simplifying:
\[
\nu_{\pm}=-\frac{1}{2}\left(1-\beta\right)\pm\frac{1}{2}\sqrt{\left(1-\beta\right){}^{2}-4\gamma\lambda_{i}}.
\]
Note that:
\[
1-\beta=2\frac{\sqrt{\mu}}{\sqrt{L}+\sqrt{\mu}}.
\]
The choice of step size ensures that the discriminant is always non-positive:
\[
\left(\frac{2\sqrt{\mu}}{\sqrt{L}+\sqrt{\mu}}\right)^{2}-4\frac{4\lambda_{i}}{\left(\sqrt{L}+\sqrt{\mu}\right)^{2}}\leq\left(\frac{2\sqrt{\mu}}{\sqrt{L}+\sqrt{\mu}}\right)^{2}-4\left(\frac{2\sqrt{\mu}}{\sqrt{L}+\sqrt{\mu}}\right)^{2}\leq0
\]
Therefore the decay rate is bounded by the real part of the eigenvalues,
which is $-\frac{1}{2}(1-\beta).$
\end{proof}

\end{document}